\documentclass[10pt,journal,compsoc]{IEEEtran}
\IEEEoverridecommandlockouts

\usepackage{tikz}
\usepackage{amsmath}

\usepackage{filecontents}
\usepackage[T1]{fontenc}
\usepackage{cite}
\usepackage{amsmath,amssymb,amsthm,amsfonts}
\usepackage{algorithmic}
\usepackage{graphicx}
\usepackage{textcomp}
\usepackage{xcolor}
\def\BibTeX{{\rm B\kern-.05em{\sc i\kern-.025em b}\kern-.08em
    T\kern-.1667em\lower.7ex\hbox{E}\kern-.125emX}}

\usepackage{hyperref}
\usepackage{url}
\usepackage{breakurl}
\usepackage{booktabs}
\usepackage{multirow}
\usepackage{subfig}
\usepackage{float}

\hypersetup{
    colorlinks,
    linkcolor={red},
    citecolor={red},
    urlcolor={blue}
}

\newtheorem{theorem}{Theorem}
\newtheorem{lemma}{Lemma}

\usepackage{hyperref}
\usepackage{cleveref}

\usepackage[linesnumbered,ruled,vlined,noend]{algorithm2e}

\usepackage{booktabs}
\usepackage{array}

\begin{document}

\title{Preserving Privacy and Security in \\Federated Learning}
\author{Truc~Nguyen
        and~My~T.~Thai
\IEEEcompsocitemizethanks{\IEEEcompsocthanksitem T. Nguyen and My T. Thai are with the Department of Computer \& Information Science \& Engineering, University of Florida, Gainesville,
FL, 32611.\protect\\
E-mail: truc.nguyen@ufl.edu and mythai@cise.ufl.edu
}
}

\markboth{Journal of \LaTeX\ Class Files,~Vol.~14, No.~8, August~2015}%
{Shell \MakeLowercase{\textit{et al.}}: Bare Demo of IEEEtran.cls for Computer Society Journals}



\IEEEtitleabstractindextext{%
\begin{abstract}
Federated learning is known to be vulnerable to both security and privacy issues. Existing research has focused either on preventing poisoning attacks from users or on concealing the local model updates from the server, but not both. However, integrating these two lines of research remains a crucial challenge since they often conflict with one another with respect to the threat model. 

In this work, we develop a principle framework that offers both privacy guarantees for users and detection against poisoning attacks from them. With a new threat model that includes both an honest-but-curious server and malicious users, we first propose a secure aggregation protocol using homomorphic encryption for the server to combine local model updates in a private manner. Then, a zero-knowledge proof protocol is leveraged to shift the task of detecting attacks in the local models from the server to the users. The key observation here is that the server no longer needs access to the local models for attack detection. Therefore, our framework enables the central server to identify poisoned model updates without violating the privacy guarantees of secure aggregation. 

\end{abstract}
}

\maketitle

\IEEEdisplaynontitleabstractindextext
\newcommand{\calU}{{\cal U}}
\newcommand{\calE}{{E}}
\newcommand{\calD}{{\cal D}}

\IEEEpeerreviewmaketitle

\IEEEraisesectionheading{\section{Introduction}\label{sec:introduction}}


Federated learning is an engaging framework for large-scale distributed training of deep learning models with thousands to millions of users. In every round, a central server distributes the current global model to a random subset of users. Each of the users trains locally and submits a model update to the server. Then, the server averages the updates into a new global model. Federated learning has inspired many applications in many domains, especially training image classifiers and next-word predictors on users’ smartphones \cite{hard2018federated}. To exploit a wide range of training data while maintaining users’ privacy, federated learning by design has no visibility into users’ local data and training.

Despite the great potential of federated learning in large-scale distributed training with thousands to millions of users, the current system is still vulnerable to certain privacy and security risks. First, although the training data of each user is not disclosed to the server, the model update is. This poses a privacy threat, as it is suggested that a trained neural network’s parameters enable a reconstruction and/or inference of the original training data \cite{nguyen2023active,fredrikson2015model,abadi2016deep}. Second, federated learning is generally vulnerable to model poisoning attacks that leverage the fact that federated learning enables adversarial users to directly influence the global model, thereby allowing considerably more powerful attacks \cite{bagdasaryan2020backdoor}. Recent research on federated learning focuses on either improving the privacy of model updates \cite{bonawitz2017practical,aono2017privacy,zhang2020batchcrypt} or preventing certain poisoning attacks from malicious users \cite{bagdasaryan2020backdoor,wang2020attack,xie2019dba}, but not both.

To preserve user privacy, secure aggregation protocols have been introduced into federated learning to devise a training framework that protects the local model updates \cite{bonawitz2017practical,aono2017privacy,zhang2020batchcrypt}. These protocols enable the server to privately combine the local models in order to update the global model without learning any information about each individual local model. Specifically, the server can compute the sum of the parameters of the local models while not having access to the local models themselves. As a result, the local model updates are concealed from the server, thereby preventing the server from exploiting the updates of any user to infer their private training data. 

On the other hand, previous studies also unveil a vulnerability of federated learning when it comes to certain adversarial attacks from malicious users, especially poisoning attacks \cite{bagdasaryan2020backdoor,wang2020attack}. This vulnerability leverages the fact that federated learning gives users the freedom to train their local models. Specifically, any user can train in any way that benefits the attack, such as arbitrarily modifying the weights of their local models. To prevent such attacks, a conventional approach is to have the central server run some defense mechanisms to inspect each of the model updates, such as using anomaly detection algorithms to filter out the poisoned ones \cite{bagdasaryan2020backdoor}.

However, combining these two lines of research on privacy and security is not trivial since they contradict one another. In particular, by allowing the server to inspect local model updates from users to filter out the attacked ones, it violates the security model of secure aggregation. In fact, under a secure aggregation setting, the server should not be able to learn any information about individual model updates. Therefore, with secure aggregation, the server cannot run any defense mechanism directly on each model update to deduce whether it is attack-free or not. For that reason, existing FL systems with secure aggregation prevent the server from detecting poisoned local models.

To tackle this issue, in this paper, we propose a framework that integrates secure aggregation with defense mechanisms against poisoning attacks from users. To circumvent the aforementioned problem, we shift the task of running the defense mechanism onto the users, and then each user is given the ability to attest the execution of the defense mechanism with respect to their model update. Simply speaking, the users are the ones who run the defense mechanism, then they have to prove to the server that the mechanism was executed correctly. To achieve this, the framework leverages a zero-knowledge proof (ZKP) protocol that the users can use to prove the correctness of the execution of a defense mechanism. The ZKP protocol must also make it difficult to generate a valid proof if the defense mechanism was not run correctly, and the proof must not reveal any information about the local model update in order the retain the privacy guarantee of secure aggregation. Via our ZKP protocol, the server can verify if a defense mechanism was properly executed by the users.

To demonstrate the feasibility of the framework, we use the backdoor attack as a use case in which we construct a ZKP protocol for a specific backdoor defense mechanism \cite{liu2018fine}. Furthermore, we also propose a secure aggregation protocol in federated learning to address the limitations of previous work. Specifically, we show that our proposed aggregation protocol is robust against malicious users in a way that it can still maintain privacy and liveness despite the fact that some users may not follow the protocol honestly. Our framework can then combine both the ZKP and the secure aggregation protocols to tackle the above-mentioned privacy and security risks in federated learning.

\noindent\textbf{Contributions.} Our main contributions are as follows:
\begin{itemize}
    \item We establish a framework integrating both the secure aggregation protocol and defense mechanisms against poisoning attacks without violating any privacy guarantees.
    \item We propose a new secure aggregation protocol using homomorphic encryption for federated learning that can tolerate malicious users and a semi-honest server while maintaining both privacy and liveness.
    \item We construct a ZKP protocol for a backdoor defense to demonstrate the feasibility of our framework.
    \item Finally, we analyze the computation and communication cost of the framework and provide some benchmarks regarding its performance.
\end{itemize}

\noindent\textbf{Organization.} The rest of the paper is structured as follows. Section \ref{sec:sec} provides the system and security models. The main framework for combining both the secure aggregation protocol and defense mechanisms against poisoning attacks is shown in Section \ref{sec:frame}. Section \ref{sec:agg} presents our proposed secure aggregation protocol for federated learning. Section \ref{sec:zkp} gives the ZKP protocol for a backdoor defense mechanism. We evaluate the complexity and benchmark our solution in Section \ref{sec:eval}. We discuss some related work in Section \ref{sec:rel} and provide some concluding remarks in Section \ref{sec:con}.

\section{System and Security model} \label{sec:sec}
Current systems proposed to address either the poisoning attacks from users or the privacy of model updates are designed using security models that conflict with each other. In fact, when preserving the privacy of model updates, the central server and the users are all considered as honest-but-curious (or semi-honest). However, when defending against poisoning attacks from users, the server is considered honest while the users are malicious. 

Due to that conflict, this section establishes a more general security model combining the ones proposed in previous work on the security and privacy of federated learning. From the security model, we also define some design goals.

\begin{figure*}
    \centering
    \includegraphics[width=0.9\linewidth]{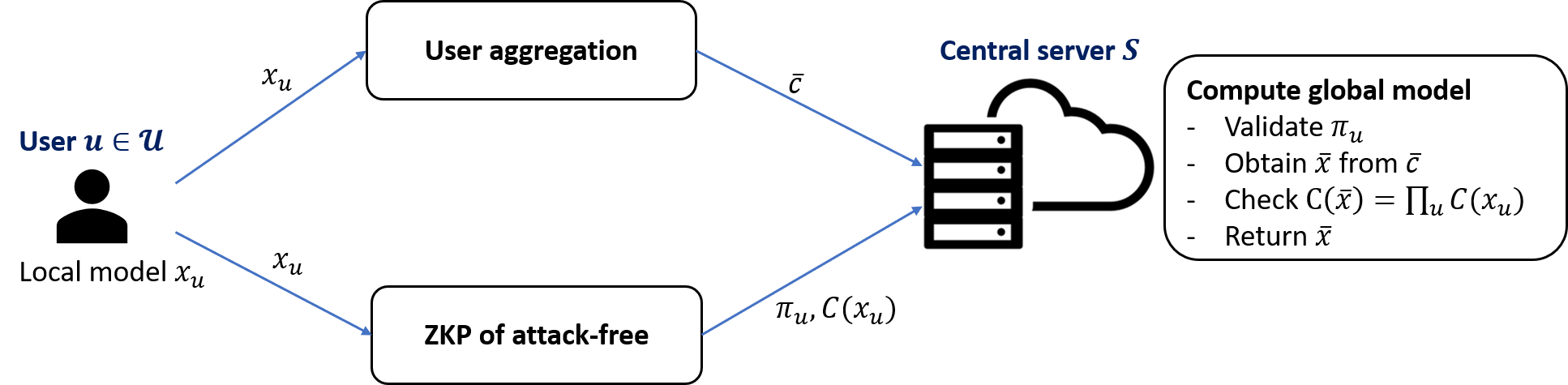}
    \caption{Secure framework for FL. Each user $u\in \calU$ trains a local model $x_u$ that is used as an input to (1) \textbf{User aggregation} and (2) \textbf{ZKP of attack-free model}. The \textbf{User aggregation} component returns $\Bar{c} = E_{pk}(\sum_{u\in \calU'} x_u)$ which is the encryption of the sum over the local models of honest users. The \textbf{ZKP of attack-free model} component returns the proof $\pi_u$ and the commitment $C(x_u)$ for each $x_u$. The outputs of these two component are then used by the central server as inputs to the \textbf{Computing global model} component. This component validates $\{\pi_u\}_{u\in\calU}$, obtains the global model $\Bar{x}$ from $\Bar{c}$, checks if $C(\Bar{x})$ is consistent with $\{C(x_u)\}_{u\in\calU}$, and then returns $\Bar{x}$.}
    \label{fig:framework}
\end{figure*}

\subsection{System model}
We define two entities in our protocol: 

\begin{enumerate}
    \item A set of $n$ users $\mathcal{U}$: Each user has a local training dataset, computes a model update based on this dataset, and collaborates with each other to securely aggregate the model updates.
    \item A single server $S$: Receives the aggregated model updates from the users and computes a global model that can later be downloaded by the users.
\end{enumerate}

We denote the model update of each user $u\in \mathcal{U}$ as an input vector $x_u$ of dimension $m$. For simplicity, all values in both $x_u$ and $\sum_{u\in \mathcal{U}} x_u$ are assumed to be integers in the range $[0, N)$ for some known $N$. The protocol is correct if $S$ can learn $\Bar{x} = \sum_{u\in \calU'} x_u$ for some subset of users $\calU' \subseteq \calU$. We assume a federated learning process that follows the FedAvg framework \cite{mcmahan2017communication} which uses a horizontal setting for FL. We denote $\calU'$ as a random subset of  $\calU$ (the server selects a random subset of users for each training round). 


In addition to securely computing $\Bar{x}$, each user $u\in \mathcal{U}$ computes a zero-knowledge proof $\pi_u$ proving that $x_u$ is a clean model and not leaking any information about $x_u$. This proof $\pi_u$ is then sent to and validated by the server $S$. The server then verifies that $\{\pi_u\}_{u \in U}$ is consistent with the obtained $\Bar{x}$.

\subsection{Security model}

\noindent\textbf{Threat model.} We take into account a computationally bounded adversary that can corrupt the server or a subset of users in the following manners. The server is considered honest-but-curious in a way that it behaves honestly according to the training protocol of federated learning but also attempts to learn as much as possible from the data received from the users.
The server can also collude with corrupted users in $\calU$ to learn the inputs of the remaining users. Moreover, some users try to inject unclean model updates to poison the global model and they can also arbitrarily deviate from the protocol.

As regards the security of the key generation and distribution processes, there has been extensive research in constructing multi-party computation (MPC) protocols where a public key is computed collectively and every user computes private keys \cite{damgaard2001practical,nishide2010distributed,das2022practical}. Such a protocol can be used in this architecture to securely generate necessary keys for the users and the server. Hence, this paper does not focus on securing the key generation and distribution processes but instead assumes that the keys are created securely and honest users' keys are not leaked to others.



\noindent\textbf{Design goals.} Besides the security aspects, our design also aims to address unique challenges in federated learning, including operating on high-dimensional vectors and users dropping out. With respect to the threat model, we summarize the design goals as follows:

\begin{enumerate}
    \item The server learns nothing except what can be inferred from $\Bar{x}$. 
    \item Each user does not learn anything about the others' inputs.
    \item The server can verify that $\Bar{x}$ is the sum of clean models.
    \item The protocol operates efficiently on high-dimensional vectors
    \item  Robust to users dropping out.
    \item Maintain undisrupted service (i.e., liveness) in case a small subset of users deviate from the protocol.
\end{enumerate}



\section{Secure Framework for Federated Learning} \label{sec:frame}
To address the new security model proposed in Section \ref{sec:sec}, we propose a framework to integrate secure aggregation with defenses against poisoning attacks in a way that they can maintain their respective security and privacy properties. Our framework is designed to ensure the privacy of users' local model while preventing certain attacks from users. As shown in Fig. \ref{fig:framework}, our framework consists of three components: (1) users aggregation, (2) zero-knowledge proof of attack-free model, and (3) computing global model. We discuss the properties of each component as follows.

\subsection{Users aggregation}
This component securely computes the sum of local models from the users. First, each user $u \in \calU$ trains a local model update $x_u$ and uses it as input to this component. Then, the component outputs $\Bar{c} = E_{pk}(\sum_{u\in \calU'} x_u)$ which is an encryption over the sum of the model updates of all honest users $\calU' \subseteq \calU$. $E_{pk}(\cdot)$ denotes a viable additive homomorphic encryption scheme, and we require that it must have the property of indistinguishability under chosen plaintext attack (IND-CPA). The encryption function will be discussed in detail in Section \ref{sec:agg}.

Furthermore, this aggregation must maintain the privacy with respect to each local model $x_u$, particularly, we should be able to simulate the output of each user by a Probabilistic Polynomial Time (PPT) simulator. We specify the security properties of the Users aggregation component in Fig. \ref{fig:ua}.

\begin{figure}[h]
	\noindent\fbox{%
		\parbox{\columnwidth}{%
			\begin{center}
				User aggregation
			\end{center}
			
			\noindent \textbf{Inputs:} 
            \begin{itemize}
                \item Local model $x_u$ of each user $u\in \calU$
                \item Public key $pk$ 
            \end{itemize}
            \noindent \textbf{Output:} $\Bar{c} = E_{pk}(\sum_{u\in \calU'} x_u)$ where $\calU' \subseteq \calU$ is the set of honest users \\
            \noindent \textbf{Properties:}
            \begin{itemize}
                \item The encryption algorithm $E(\cdot)$ is additively homomorphic and IND-CPA secure.
                \item $x_u$ is not revealed to other users $\calU \setminus \{u\}$ nor the server $S$. More formally, there exists a PPT simulator that can simulate each $u\in \calU$ such that it can generate a view for the adversary that is computationally indistinguishable from the adversary's view when interacting with honest users.
                \item Robust to users dropping out
            \end{itemize}
		}
	}
	\caption{User aggregation component}
	\label{fig:ua}
\end{figure}


\subsection{Zero-knowledge proof of attack-free model}
This component generates a zero-knowledge proof proving the local model is free from a certain poisoning attack. In specific, we first assume the existence of a function that verifies an attack-free model as follows:

\begin{equation}
    R(x) = 
    \begin{cases}
    C(x) \qquad \text{if $x$ is attack-free}\\
    \bot \qquad \text{otherwise}
    \end{cases}
\end{equation}
where $x$ is a local model, $C(x)$ is the commitment of $x$. Next, the component implements a zero-knowledge proof protocol for the users to attest the correct execution of $R(x)$ without revealing $x$. Note that the ZKP itself is not a defense mechanism, but it is a means to verify that the users have executed the function $R(\cdot)$ on their side. Hence, the effectiveness of this component in preventing poisoning attacks depends on the performance of the defense $R(\cdot)$. A specific implementation of this component for defending against backdoor attacks is shown in Section \ref{sec:zkp}. The properties of this component are shown in Fig. \ref{fig:zk}.

\begin{figure}
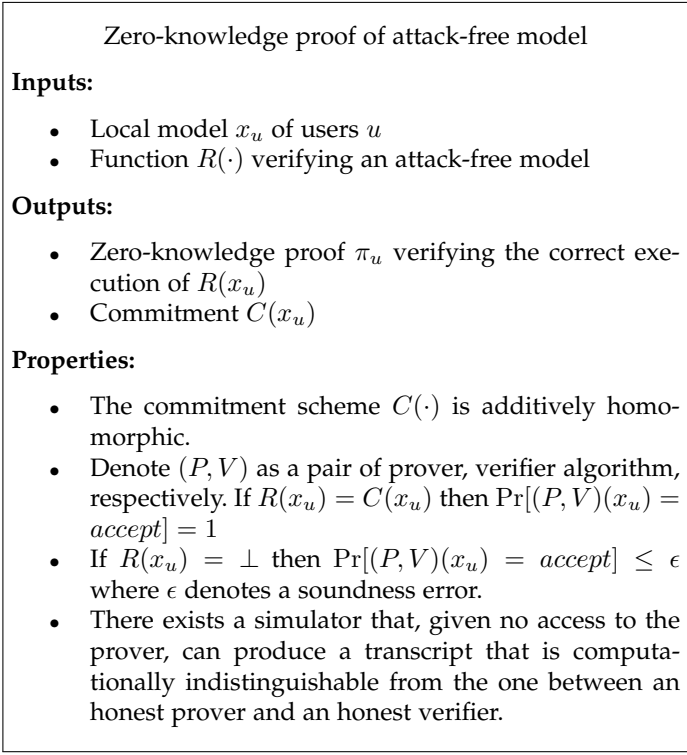

	\noindent\fbox{%
		\parbox{\columnwidth}{%
			\begin{center}
				Zero-knowledge proof of attack-free model
			\end{center}
			
			\noindent \textbf{Inputs:} 
            \begin{itemize}
                \item Local model $x_u$ of users $u$
                \item Function $R(\cdot)$ verifying an attack-free model
            \end{itemize}
            \noindent \textbf{Outputs:} 
            \begin{itemize}
                \item Zero-knowledge proof $\pi_u$ verifying the correct execution of $R(x_u)$
                \item Commitment $C(x_u)$
            \end{itemize}
            \noindent \textbf{Properties:}
            \begin{itemize}
                \item The commitment scheme $C(\cdot)$ is additively homomorphic.
                \item Denote $(P,V)$ as a pair of prover, verifier algorithm, respectively. If $R(x_u) = C(x_u)$ then $\Pr[(P,V)(x_u) = accept] = 1$
                \item If $R(x_u) = \bot$ then $\Pr[(P,V)(x_u) = accept] \leq \epsilon$ where $\epsilon$ denotes a soundness error.
                \item There exists a simulator that, given no access to the prover, can produce a transcript that is computationally indistinguishable from the one between an honest prover and an honest verifier.
            \end{itemize}
		}
	}
	\caption{Zero-knowledge proof component}
	\label{fig:zk}
\end{figure}

The prover algorithm $P$ is run by the users, while the verifier algorithm $V$ is run by the server $S$.

\subsection{Computing global model}
This component takes as input the output of the other two components and produces the corresponding global model. Hence, it serves to combine two previous components and it is run solely by the server. Specifically, it needs to be able to decrypt $\Bar{c}$ from the first component and to validate the zero-knowledge proofs $\{\pi_u\}_{u\in \calU'}$ from the second component. Then, it outputs the sum of the model updates from the honest users. We define this component in Fig. \ref{fig:glb}.

\begin{figure}
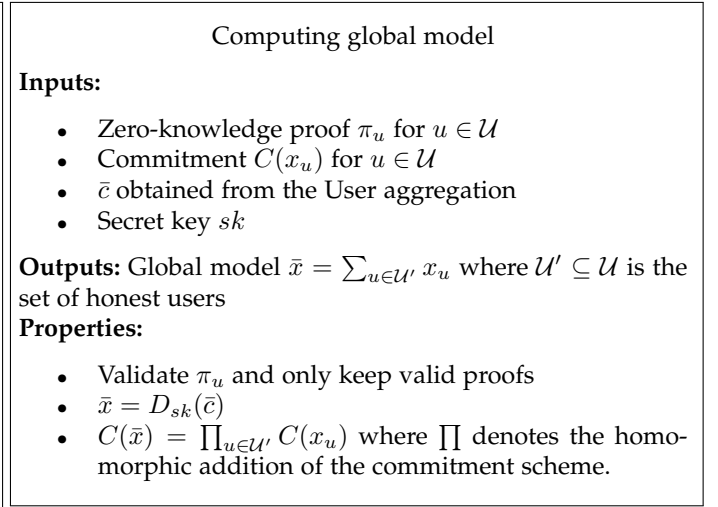

	\noindent\fbox{%
		\parbox{\columnwidth}{%
		    \begin{center}
				Computing global model
			\end{center}
			
			\noindent \textbf{Inputs:} 
            \begin{itemize}
                \item Zero-knowledge proof $\pi_u$ for $u\in \calU$
                \item Commitment $C(x_u)$ for $u\in \calU$
                \item $\Bar{c}$ obtained from the User aggregation
                \item Secret key $sk$
            \end{itemize}
            \noindent \textbf{Outputs:} Global model $\Bar{x} = \sum_{u\in \calU'} x_u$ where $\calU' \subseteq \calU$ is the set of honest users\\
            \noindent \textbf{Properties:}
            \begin{itemize}
                \item Validate $\pi_u$ and only keep valid proofs
                \item $\Bar{x} = D_{sk}(\Bar{c})$
                \item $C(\Bar{x}) = \prod_{u\in \calU'}C(x_u)$ where $\prod$ denotes the homomorphic addition of the commitment scheme.
            \end{itemize}
		}
	}
	\caption{Computing global model component}
	\label{fig:glb}
\end{figure}

From the third component, we can prove that this framework preserves the privacy of users' model updates given that all the components satisfy their properties. Let $Adv$ be a random variable representing the joint views of all parties that are corrupted by the adversary according to the threat model. We show that it is possible to simulate the combined view of any subset of honest-but-curious parties given only the inputs of those parties and the global model $\Bar{x}$. This means that those parties do not learn anything other than their own inputs and $\Bar{x}$.

\begin{theorem}
There exists a PPT simulator $\mathcal{S}$ such that for  ${Adv} \subseteq \mathcal{U} \cup S$, the output of the simulator $\mathcal{S}$ is computationally indistinguishable from the output of $Adv$
\end{theorem}
\begin{proof}
Refer to Appendix A.
\end{proof}

\section{Secure Aggregation in Federated Learning} \label{sec:agg}
In this section, we describe a secure aggregation scheme that realizes the User aggregation component of our framework. Although some secure aggregation schemes have been proposed for federated learning, they are not secure against our threat model in Section \ref{sec:sec}. Particularly, the aggregation schemes in \cite{aono2017privacy} and \cite{zhang2020batchcrypt} assume no collusion between the server and a subset of users. The work of Bonawitz et al. \cite{bonawitz2017practical} fails to reconstruct the global model if one malicious user arbitrarily deviates from the secret sharing protocol. A detailed discussion of these related work is provided in Section \ref{sec:rel}.

The main challenge here is that we want to provide robustness in spite of malicious users, i.e., the protocol should be able to tolerate users that do not behave according to the protocol and produce correct results based on honest users. Moreover, we need to minimize the users' communication overhead when applying this scheme.

\subsection{Cryptographic primitives}

\noindent\textbf{Damgard-Jurik and Paillier cryptosystems \cite{paillier1999public,damgaard2010generalization}.}
Given $pk = (n, g)$ where $n$ is the public modulus and $g \in \mathbb{Z}^*_{n^{s+1}}$ is the base, the Damgard-Jurik encryption of a message $m \in \mathbb{Z}_{n^s}$ is $\calE_{pk}(m)=g^{m}r^{n^s} \bmod n^{s+1}$ for some random $r \in \mathbb{Z}^*_{n}$ and a positive natural number $s$. The Damgard-Jurik cryptosystem is additively homomorphic such that $\calE_{pk}(m_1) \cdot \calE_{pk}(m_2) = \calE_{pk}(m_1 + m_2)$. This cryptosystem is semantically secure (i.e., IND-CPA) based on the assumption that computing n-th residue classes is believed to be computationally difficult. The Paillier's scheme is a special case of Damgard-Jurik with $s=1$.

In the threshold variant of these cryptosystems \cite{damgaard2010generalization}, several shares of the secret key are distributed to some decryption parties, such that any subset of at least $t$ of them can perform decryption efficiently, while less than $t$ cannot decrypt the ciphertext. Given a ciphertext, each party can obtain a partial decryption. When $t$ partial decryptions are available, the original plaintext can be obtained via a \emph{Share combining} algorithm.

\noindent\textbf{Pedersen commitment \cite{pedersen1991non}.}
A general commitment scheme includes an algorithm $C$ that takes as inputs a secret message $m$ and a random $r$, and outputs the commitment $commit$ to the message $m$, i.e., $commit = C(m, r)$. A commitment scheme has two important properties: (1) $commit$ does not reveal anything about $m$ (i.e., hiding), and (2) it is infeasible to find $(m',r') \neq (m, r)$ such that $C(m', r') = C(m,r)$ (i.e., binding). Pedersen commitment \cite{pedersen1991non} is a realization of such scheme and it is also additively homomorphic such that $C(m,r) \cdot C(m',r') = C(m + m', r + r')$

\subsection{Secure aggregation protocol} \label{ssec:secagg}

Our secure aggregation scheme relies on additive homomorphic encryption schemes, that is $\calE_{pk}(m_1 + m_2) = \calE_{pk}(m_1) \cdot \calE_{pk}(m_2)$. In this work, we use the Daamgard-Jurik public-key cryptosystem \cite{damgaard2010generalization} which is both IND-CPA secure and additively homomorphic. The cryptosystem supports threshold encryption and operates on the set $\mathbb{Z}^*_n$.

We also assume that the communication channel between each pair of users is secured with end-to-end encryption, such that an external observer cannot eavesdrop the data being transmitted between the users. The secure communication channel between any pair of users can be mediated by the server.

In the one-time setup phase, the secret shares $\{s_u\}_{u\in \calU}$ of the private key $sk$ are distributed among $n$ users in $\calU$ via the threshold encryption scheme. The scheme requires a subset $\calU_t \subset \calU$ of at least $t$ users to perform decryption with $sk$. This process of key distribution can be done via a trusted dealer or via an MPC protocol. After that, each secure aggregation process is divided into 2 phases: (1) Encryption and (2) Decryption.


\noindent\textbf{Encryption phase:} To securely compute the sum $\Bar{x} = \sum_{u\in \calU} x_u$, each user $u\in \calU$ first computes the encryption $c_u = \calE_{pk}(x_u)$ and then send $c_u$ to the server. The server then computes the product of all $\{c_u\}_{u\in \calU}$, denoting as $\Bar{c}$:

\begin{equation}
\begin{aligned}
    \Bar{c} = \prod_{u\in \calU}c_u \mod{n^{s+1}}
    &= \prod_{u\in \calU} \calE_{pk}(x_u) \mod{n^{s+1}}\\
    &= \calE_{pk}\left(\sum_{u\in \calU}x_u \mod{n^{s+1}}\right)
\end{aligned}  
\end{equation}

Additionally, per IND-CPA, knowing the ciphertext $c_u$ implies nothing about the plaintext $x_u$ for $u\in \calU$. As each user only sends out $c_u$, neither the server $S$ nor other entities $\calU \setminus \{u\}$ with access to $c_u$ can infer any information about $x_u$. This satisfies the security properties that are specified in Fig. \ref{fig:ua}. 

\noindent\textbf{Decryption phase:} Suppose that up to this phase, some amount of less than $N-t$ users dropout, that means we still have the subset $\calU_t \subset \calU$ of at least $t$ users remaining to perform decryption. 

First, the server sends $\Bar{c}$ to all users in $\calU_t$. Each user $u \in \calU_t$ computes $c_u' = \Bar{c}^{2s_u\Delta}$ where $\Delta = t!$, and then sends $c_u'$ back to the server. The server then applies the \emph{Share combining} algorithm in \cite{damgaard2010generalization} to compute $\Bar{x} = \sum_{u\in \calU} x_u$ from $\{c_u'\}_{u\in\calU_t}$. As a result, the server can securely obtain $\Bar{x}$ without knowing each individual $x_u$.

\noindent\textbf{User drop-outs.} Furthermore, by design, it can be seen from this protocol that we can tolerate up to $N-t$ users dropping out and the server is still able to construct the global model of all users. In other words, as long as there are at least $t$ users remaining, it will not affect the outcome of the protocol. Therefore, our protocol is robust to user drop-outs up to $N-t$.

\noindent\textbf{Security against active adversary.} We analyze the security in the active adversary model. First, the protocol is secure in spite of collusion between the server and a subset of users in a way that allows the server to decrypt some $c_u$ to obtain the corresponding $x_u$. By using the threshold encryption, as long as there is at least one honest user in $\calU_t$, the server will not be able to learn $x_u$ of any user. Hence, the protocol can tolerate up to $t-1$ malicious users, with $t<|\calU|$.

Second, in the encryption phase, if a user deviates from the protocol, e.g., by computing $c_u$ as a random string or using a different key other than $pk$, the server will not be able to decrypt $\Bar{c}$. To tolerate this behavior, the server can iteratively drop each $c_u$ from $\Bar{c}$ until it obtains a decryptable value. 

	
	
	
      

In the decryption phase, if a user $u^*$ does not cooperate to provide $c_{u^*}'$, then the server will not have all the necessary information to perform decryption. However, if we still have at least $t$ honest users (not counting $u^*$), we can eliminate $u^*$ from $\calU_t$ and add another user to $\calU_t$, hence, the server can still perform decryption.

In summary, our secure aggregation protocol can tolerate $t-1$ malicious users and be robust against $|\calU|-t$ users dropping out (where $t<|\calU|$). Therefore, $t$ should be determined by calibrating the trade-off between the usability (small $t$) and security (large $t$).

\subsection{Efficient encryption of $x_u$}


We discuss two methods for reducing the ciphertext size when encrypting $x_u$. Since $x_u$ is a vector of dimension $m$, naively encrypting every element of $x_u$ using either the Paillier or Damgard-Jurik cryptosystems will result in a huge ciphertext with the expansion factor of $2\log n$ (i.e., the ratio between the size of $c_u$ and of $x_u$). Below, we present (1) a polynomial packing method and (2) a hybrid encryption method that includes a key encapsulation mechanism (KEM) and a data encapsulation mechanism (DEM).

\noindent\textbf{Polynomial packing.} We can reduce the expansion factor by packing $x_u$ into polynomials as follows. Suppose $x_u \equiv (x_u^{(1)}, x_u^{(2)}, ..., x_u^{(m)})$, we transform this into:

\begin{equation}
    x_u' = x_u^{(1)} + x_u^{(2)} \cdot 2^b + x_u^{(3)} \cdot 2^{2b} ... + x_u^{(m)} \cdot 2^{(m-1)b}
\end{equation}
where $b = \lceil \log N \rceil$. Hence, $x_u'$ becomes a $mb$-bit number.

Then, to encrypt $x_u$, we simply run the Damgard-Jurik encryption on $x_u'$. Specifically, we obtain the ciphertext $c_u = \calE(x_u')$. This $c_u$ will have $(\lceil \frac{mb}{\log n} \rceil + 1)\log n $ bits. The expansion factor can be calculated as:

\begin{equation}
    \lim_{m \rightarrow \infty} \frac{(\lceil \frac{mb}{\log n} \rceil + 1)\log n }{mb} = 1
\end{equation}

Therefore, as $m$ increases, i.e., larger models, the expansion factor approaches 1, thereby implying minimal overhead. Note that since we are packing into polynomials, the additive homomorphic property of the ciphertexts is retained.

Additionally, if we use this polynomial packing scheme with the Paillier encryption, it would result in an expansion factor of 2 as Paillier requires the plaintext size to be smaller than the key size. To keep the expansion factor of 1 when using Paillier, we introduce a KEM-DEM scheme as follows.

\noindent\textbf{Hybrid encryption (KEM-DEM technique).} In \cite{canteaut2018stream}, the authors propose a KEM-DEM approach to compress the ciphertext of Fully Homomorphic Encryption (FHE). With FHE, they can use blockciphers to generate the keystreams. However, since our work only deals with Partially HE (PHE) like Paillier, we cannot use blockciphers, hence, we need to propose a different approach to generate the keystream with only PHE.


The general idea of our protocol is as follows. Suppose a user wants to send $E_{pk}(x)$ to the server (for simplicity, we use the notation $x$ instead of $x_u$). Instead of sending $c = E_{pk}(x)$ directly, the user chooses a secret random symmetric key $k$ that has the same size as an element in $x$ and computes $c' = (E_{pk}(k), SE_k(x))$ where $SE_k(x) = x-k$, which means subtracting $k$ from every element in $x$. $c'$ is then sent to the server. When the server receives $c'$, it can obtain $c = E_{pk}(x) $ by computing $E_{pk}(k) \times E_{pk}(SE_k(x)) = E_{pk}(x) = c$.


The correctness of this scheme holds because the server can obtain $E_{pk}(x)$ from the user. The security also holds because the server does not know the key $k$ chosen by the user, hence, it cannot know $x$. However, this scheme is not semantically secure, i.e., IND-CPA secure, because by defining $SE_k(x) = x - k$, identical elements in $x$ will result in the same ciphertext.

To resolve this, we use an IV (initialization vector) for the encryption function. IV is a nonce which could be a random number or a counter that is increased each time it is used. Denote $x = (x_1, ..., x_m)$, $\oplus$ as a bitwise XOR operation, on the user's side:

\begin{enumerate}
    \item  Choose the secret $l$-bit random key $k$ and a nonce IV.
\item compute the keystream $ks$: For $i\in [1,m]$, compute $ks_i = k \times (IV \oplus i) \bmod p$ where $p$ is an $l$-bit number.
\item Compute $SE_k(x) = x - ks = (x_1 - ks_1, ..., x_m - ks_m)$
\item Send $(IV, E_{pk}(k), SE_k(x))$ to the server
\end{enumerate}

\noindent On the server's side, receiving $(IV, E_{pk}(k), SE_k(x))$:
\begin{enumerate}
    \item Compute the encrypted keystream $E_{pk}(ks)$: For $i\in [1,m]$, compute $E_{pk}(k) ^{IV \oplus i} \bmod n^2 = E_{pk}(k \times (IV \oplus i)) = E_{pk}(ks_i)$
\item  Compute $E_{pk}(x) = E_{pk}(SE_k(x)) \cdot E_{pk}(ks) = E_{pk}(x_1 - ks_1 + ks_1, ..., x_m - ks_m + ks_m) = E_{pk}(x_1, ..., x_m)$
\end{enumerate}

\begin{lemma}\label{lemma:dem}
By picking $k$ uniformly at random, and $p$ as an $l$-bit prime number, the DEM component (i.e., $SE_k(x)$) is IND-CPA secure.
\end{lemma}
\begin{proof}
To show that the DEM component is IND-CPA secure, it is sufficient prove that $k \times (IV \oplus i) \bmod p$ is uniformly random with respect to $k$ (Theorem 3.32 in \cite{katz2020introduction}). In fact, in $\mathbb{Z}^*_p$, by picking $k$ uniformly at random, $k \times (IV \oplus i) \bmod p$ is uniformly distributed if and only if neither $k$ nor $(IV \oplus i)$ shares prime factors with $p$. Hence, $p$ should be an $l$-bit prime number and $k$ can be chosen randomly from $\mathbb{Z}^*_p$.
\end{proof}

Since the KEM component (i.e., $E_{pk}(k)$) is IND-CPA secure by default using Paillier, and the DEM component is also IND-CPA secure as previously shown in Lemma \ref{lemma:dem}, we can derive the following theorem:

\begin{theorem}
The KEM-DEM scheme is IND-CPA secure.
\end{theorem}

Figure \ref{fig:kemdem} illustrates this scheme. The bandwidth expansion factor of this scheme can be computed as:
\begin{equation}
    \lim_{m \rightarrow \infty} \frac{mb + 2\log n + c}{mb} = 1
\end{equation}
where $c$ is the size of the IV. Therefore, this scheme also achieves the expansion factor close to 1 as $m$ grows, however, it incurs some computational overhead as compared to the polynomial packing scheme.

\begin{figure}
    \centering
    \includegraphics[width=\linewidth]{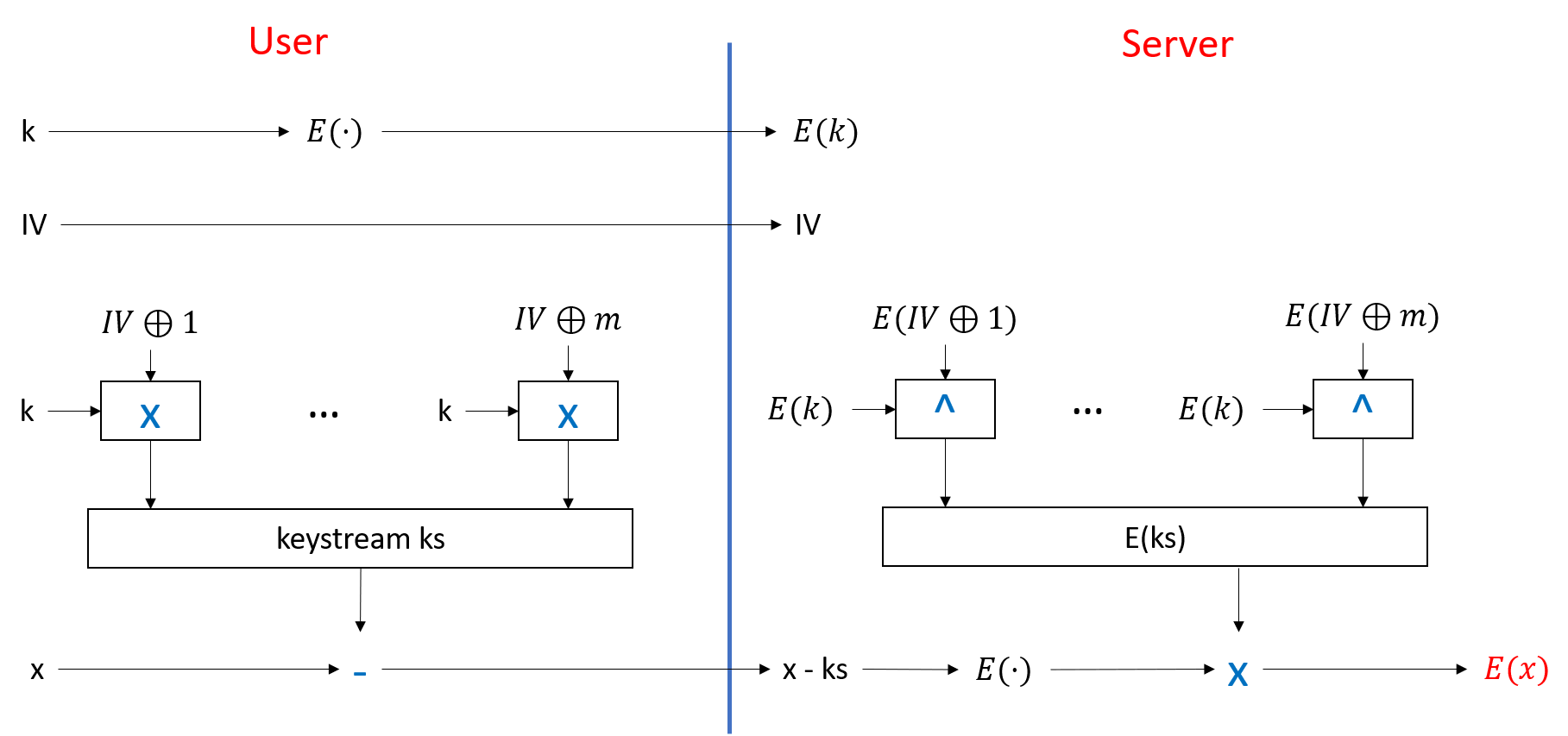}
    \caption{KEM-DEM for PHE. $\oplus$ denotes a bitwise XOR operation.}
    \label{fig:kemdem}
\end{figure}

Although both the KEM-DEM scheme and polynomial packing have the same asymptotic expansion factor (e.g., 1), each has its own advantages and disadvantages. The polynomial packing is simple and works out-of-the-box if we use the Damgard-Jurik cryptosystem. This is because Damgard-Jurik allows the plaintext size to be greater than the key size. However, it will not work with conventional cryptosystems like Paillier as they require the plaintext size to be smaller than the key size. The KEM-DEM scheme would work for all cryptosystems but it is more complex and incurs more computational overhead than the polynomial packing scheme. Since our protocol is using the Damgard-Jurik cryptosystem, Section \ref{sec:eval} only evaluates the polynomial packing scheme.

\section{Use case: ZKP for backdoor defenses} \label{sec:zkp}
This section describes an implementation that realizes the Zero-knowledge proof component of our framework. Using the backdoor attack as a use case, we provide a ZKP protocol that can detect backdoor attacks from users. Before that, we establish some background regarding multi-party protocol (MPC) and ZKP.

\subsection{Preliminaries}
\noindent\textbf{Multi-party computation (MPC) protocol.}
In a generic MPC protocol, we have $n$ players $P_1, ..., P_n$, each player $P_i$ holds a secret value $x_i$ and wants to compute $y = f(x)$ with $x = (x_1, ..., x_n)$ while keeping his input private. The players jointly run an $n$-party MPC protocol $\Pi_f$. The \emph{view} of the player $P_i$, denoted by $View_{P_i}(x)$, is defined as the concatenation of the private input $x_i$ and all the messages received by $P_i$ during the execution of $\Pi_f$. Two views are \emph{consistent} if the messages reported in $View_{P_j}(x)$ as incoming from $P_i$ are consistent with the outgoing message implied by $View_{P_j}(x)$ $(i \neq j)$. Furthermore, the protocol $\Pi_f$ is defined as perfectly correct if there are $n$ functions $\Pi_{f,1}, ..., \Pi_{f,n}$ such that $y = \Pi_{f,i}(View_{P_i}(x))$ for all $i \in [n]$. An MPC protocol is designed to be $t$-private where $1 \leq t < n$: the views of any $t$-subset of the parties do not reveal anything more about the private inputs of any other party.

\noindent\textbf{Zero-knowledge proof (ZKP).}
Suppose that we have a public function $f$, a secret input $x$ and a public output $y$. By using a ZKP protocol, a prover wants to prove that it knows an $x$ s.t. $f(x) = y$, without revealing what $x$ is. A ZKP protocol must have three properties as follows: (1) if the statement $f(x) = y$ is correct, the probability that an honest verifier accepting the proof from an honest prover is 1 (i.e., completeness), (2) if the statement is incorrect, with a probability less than some small soundness error, an honest verify can accept the proof from a dishonest prover showing that the statement is correct (i.e., soundness), and (3) during the execution of the ZKP protocol, the verifier cannot learn anything other than the fact that the statement is correct (i.e., zero-knowledge).

\noindent\textbf{Backdoor attacks and defenses.}
Backdoor attacks aim to make the target model to produce a specific attacker-chosen output label whenever it encounters a known trigger in the input. At the same time, the backdoored model behaves similarly to a clean model on normal inputs. In the context of federated learning, any user can replace the global model with another so that (1) the new model maintains a similar accuracy on the federated-learning task, and (2) the attacker has control over how the model behaves on a backdoor task specified by the attacker \cite{bagdasaryan2020backdoor}. 



To defend against backdoor attacks, Liu et al. \cite{liu2018fine} propose a \emph{pruning defense} that can disable backdoors by removing \textit{backdoor neurons} that are dormant for normal inputs. The defense mechanism works in the following manner: for each iteration, the defender executes the backdoored DNN on samples from the testing dataset, and stores the activation of each neuron per test sample. The defender then computes the neurons' average activations over all samples. Next, the defender prunes the neuron having the lowest average activation and obtains the accuracy of the resulted DNN on the testing dataset. The defense process terminates when the accuracy drops below a predetermined threshold.

When implementing this pruning defense, to prune a neuron, we simply set the weights and biases associated to that neuron to zero. In this way, we can retain the structure of the local models $x_u$.


\subsection{ZKP protocol to backdoor detection}

To inspect a local model for backdoors while keeping the model private, we propose a non-interactive ZKP such that each user $u\in \calU$ can prove that $x_u$ is clean without leaking any information about $x_u$. 

First, we devise an algorithm $backdoor$ that checks for backdoor in a given DNN model $x$ as in Algorithm \ref{algo:backdoor}.
\begin{algorithm}
	\SetAlgoLined
	\KwIn{DNN model $x$, validation dataset}
	\KwOut{1 if $x$ is backdoored, 0 otherwise}
      
    Test $x$ against the validation dataset, and record the average activation of each neuron and the accuracy\;
    $k \gets$ the neuron that has the minimum average activation\;
    Prune $k$ from $x$\;
    Test $x$ against the validation dataset and record the accuracy\;
    \If{accuracy drops by a threshold $\tau$}
    {
        \Return 1
    }
    \Else{
        \Return 0
    }
	\caption{$backdoor(x)$ -- Check for backdoors in a DNN model}
	\label{algo:backdoor}
\end{algorithm}

\begin{theorem}
$backdoor(x)$ returns 1 $\Rightarrow$ $x$ is a non-backdoored model
\end{theorem}
\begin{proof}
Refer to Appendix B.
\end{proof}

Given a model $x$, we define $R(r, x)$ for some random $r$ as follows:
\begin{equation}
    R(r,x) = 
    \begin{cases}
    C(r,x) \qquad \text{if $backdoor(x)$ returns 1}\\
    \bot \qquad \text{otherwise}
    \end{cases}
\end{equation}
where $C(r,x)$ is the commitment of $x$

Based on \cite{ishai2007zero}, we construct the ZKP protocol as follows. Each user $u\in \calU$ generates a zero-knowledge proof $\pi_u$ proving $x_u$ is clean, i.e., $R(r_u,x_u) = C(r_u,x_u)$. To produce the proof, each user $u$ runs Algorithm \ref{algo:prove} to obtain $\pi_u$. Note that $\Pi_f$ is a 3-party 2-private MPC protocol, which means a subset of any 2 views do not reveal anything about the input $x_u$. Then, $\pi_u$ and $commit_u = C(r_u,x_u)$ are sent to the server where they will be validated as in Algorithm \ref{algo:verify}.

Finally, for the server to validate that the $\Bar{x}$ obtained from secure aggregation is the sum of the models that were validated by the ZKP protocol, it collects $\{r_u\}_{u\in \calU}$ and verify that $\prod_{u\in \calU} commit_u = C\left(\sum_{u\in \calU} r_u, \Bar{x}\right)$. This is derived from the additive homomorphic property of the Pedersen commitment scheme, where

\begin{equation}
\begin{aligned}
    \prod_{u\in \calU} commit_u = \prod_{u\in \calU} C(r_u, x_u) &= C\left(\sum_{u\in \calU} r_u, \sum_{u\in \calU} x_u\right)\\
    &= C\left(\sum_{u\in \calU} r_u, \Bar{x}\right)
\end{aligned}
\end{equation}

\begin{algorithm}
	\SetAlgoLined
	\KwIn{$x_u$, $\lambda \in \mathbb{N}$, a perfectly correct and 2-private MPC $\Pi_f$ protocol among $3$ players}
	\KwOut{Proof $\pi_u$}
    
    $\pi_u \gets \emptyset$\;
    \For{$k = 1,..., \lambda$}
    {
        Samples $3$ random vectors $x_1, x_2, x_3$ such that $x_u = x_1 + x_2 + x_3$\;
        Consider an $3$-input function $f(x_1, x_2, x_3) = R(r, x_1 + x_2 + x_3)$ and emulate the protocol $\Pi_f$ on inputs $x_1, x_2, x_3$. Obtain the views $v_i = View_{P_i}(x)$ for all $i \in [3]$\;
        Computes the commitments $commit_1, ..., commit_3$ to each of the $n$ produced views $v_1, v_2, v_3$\;
        For $j \in \{1, 2\}$, computes $e_j = H(j, \{commit_i\}_{i\in [3]})$\;
        $\pi_u \gets \pi_u \cup \left(\{e_j\}_{j\in \{1,2\}},\{v_{e_j}\}_{j\in \{1,2\}},\{commit_i\}_{i\in [3]}\right)$
    }
    
    \Return $\pi_u$
    
	\caption{Generate zero-knowledge proof}
	\label{algo:prove}
\end{algorithm}

\begin{algorithm}
	\SetAlgoLined
	\KwIn{$\pi_u, commit_u$}
	\KwOut{1 if accept, 0 if reject}
    
    \For{$p_k \in \pi_u$}
    {
        $\left(\{e_j\}_{j\in \{1,2\}},\{v_{e_j}\}_{j\in \{1,2\}},\{commit_i\}_{i\in [3]}\right) \gets p_k$\;
        \If{$\{commit_i\}_{i\in [3]}$ is invalid}{\Return 0}
        \If{$e_j \neq H(j, \{commit_i\}_{i\in [3]})$ for $j \in \{1, 2\}$}{\Return 0}
        \If{$\exists e\in \{e_1, e_2\}: \Pi_{f,e}(View_{P_e}(x)) \neq commit_u$}{\Return 0}
        \If{$v_{e_1}$ and $v_{e_2}$ are not consistent with each other}{\Return 0}
    }
    
    \Return 1
    
	\caption{Validate zero-knowledge proof}
	\label{algo:verify}
\end{algorithm}

\begin{theorem}
The proposed zero-knowledge proof protocol satisfies the completeness, zero-knowledge, and soundness properties with soundness error $\frac{2}{3}^{\lambda}$ where $\lambda$ is a security parameter.
\end{theorem}

\begin{proof}
Refer to Appendix C.
\end{proof}

\section{Evaluation} \label{sec:eval}
In this section, we analyze the computation and communication cost of the proposed protocol on both the users and the server. Then we implement a proof-of-concept and benchmark the performance.


\subsection{Performance analysis}
\noindent\textbf{Users.} For the communication cost, w.r.t the secure aggregation, a user has to send out the encrypted model, which cost $O(m)$. Regarding the zero-knowledge proof, the user needs to submit a proof of size $O(mk\lambda)$ where $k$ is the size of the commitments. 

As regards the computation cost, the computation includes packing the data vector into polynomials and perform encryption which take $O(m+s)$ in total for the secure aggregation. The zero-knowledge proof protocol requires $\lambda$ runs of $O(L)$ model inferences where $L$ is the size of the validation dataset, thus the total cost is $O(\lambda L)$


\noindent\textbf{Server.} In the secure aggregation scheme, since the server receives data that are sent from all users, the communication complexity is $O(m|\calU|)$. As regards the computation cost, for the secure aggregation, the computation includes aggregating the ciphertext from each user and perform decryption on the aggregated model, which takes $O(m|\calU|^2)$ in total. Validating the zero-knowledge proofs from all users requires another $O(|\calU|k\lambda)$ computations.


\subsection{Proof-of-concept and benchmarking}
We evaluate the overhead of our proposed framework based on the running time and communication costs. To that extent, we use the following metrics 
\begin{itemize}
    \item Wall-clock running time: the total time it takes from the beginning to the end of the framework execution.
    \item Bandwidth consumption: the total amount of data sent/received by each entity in the system
    \item Bandwidth expansion factor: the ratio between the total bandwidth consumption when using the framework divided by the total bandwidth consumption when not using the framework.
\end{itemize}

\subsubsection{Secure aggregation.} To measure the performance, we implement a proof-of-concept in C\texttt{++} and Python 3. The Damgard-Jurik cryptosystem is used with 1024-, 2048-, 3072-, and 4096-bit keys. To handle big numbers efficiently, we use the GNU Multiple Precision Arithmetic Library \cite{granlund2010gnu}. We assume that each element in $\sum_{u\in \calU} x_u$ can be stored with 3 bytes without overflow. This assumption conforms with the experiments in \cite{bonawitz2017practical}. All the experiments are conducted on a single-threaded machine equipped with an Intel Core i7-8550U CPU and 16GB of RAM running Ubuntu 20.04.

\begin{figure}
    \centering
    \includegraphics[width=0.7\linewidth]{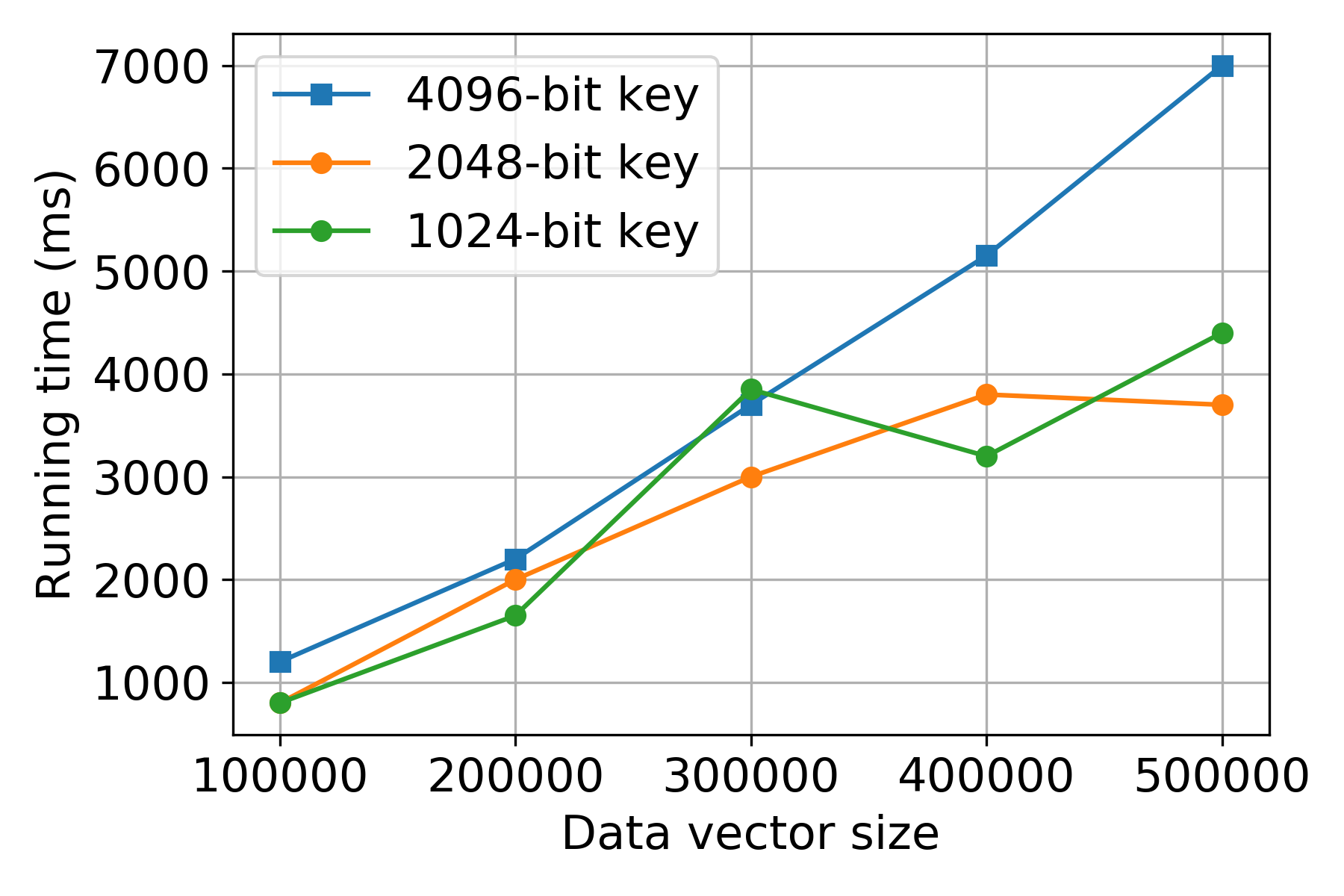}
    \caption{Wall-clock running time per user, as the size of the
data vector increases.}
    \label{fig:time}
\end{figure}



\begin{figure}
    \centering
	\subfloat[Total bandwidth expansion factor per user, as compared to sending the raw data vector to the server. Different lines represent different values of key sizes $k$]{
	    \includegraphics[width=0.46\linewidth]{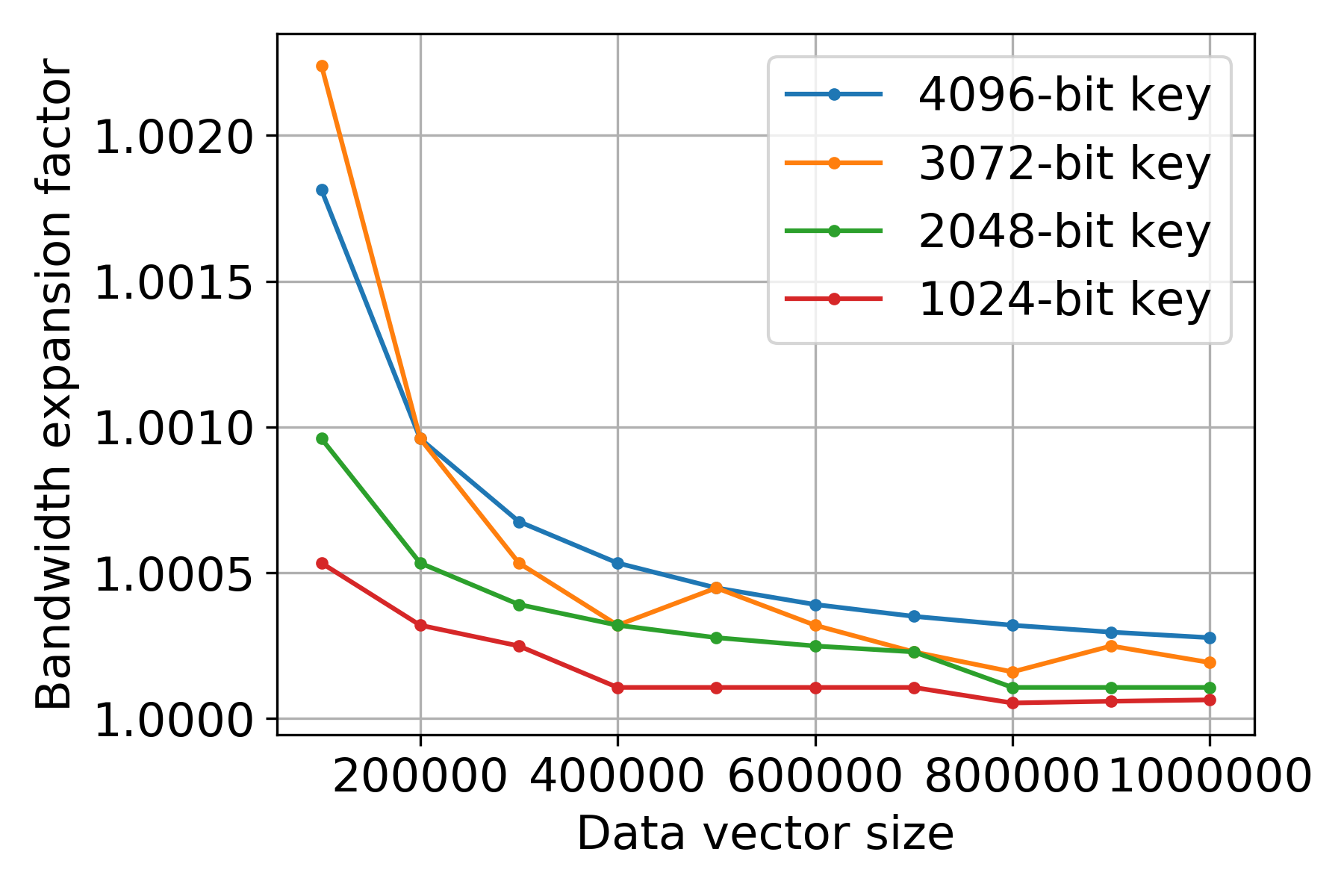}
    	\label{fig:factor}
	}%
	\hfill
	\subfloat[Total data transfer per user. The amount of transmitted data at 1024-bit, 2048-bit, and 3072-bit keys are nearly identical to 4096-bit key]{
	    \includegraphics[width=0.46\linewidth]{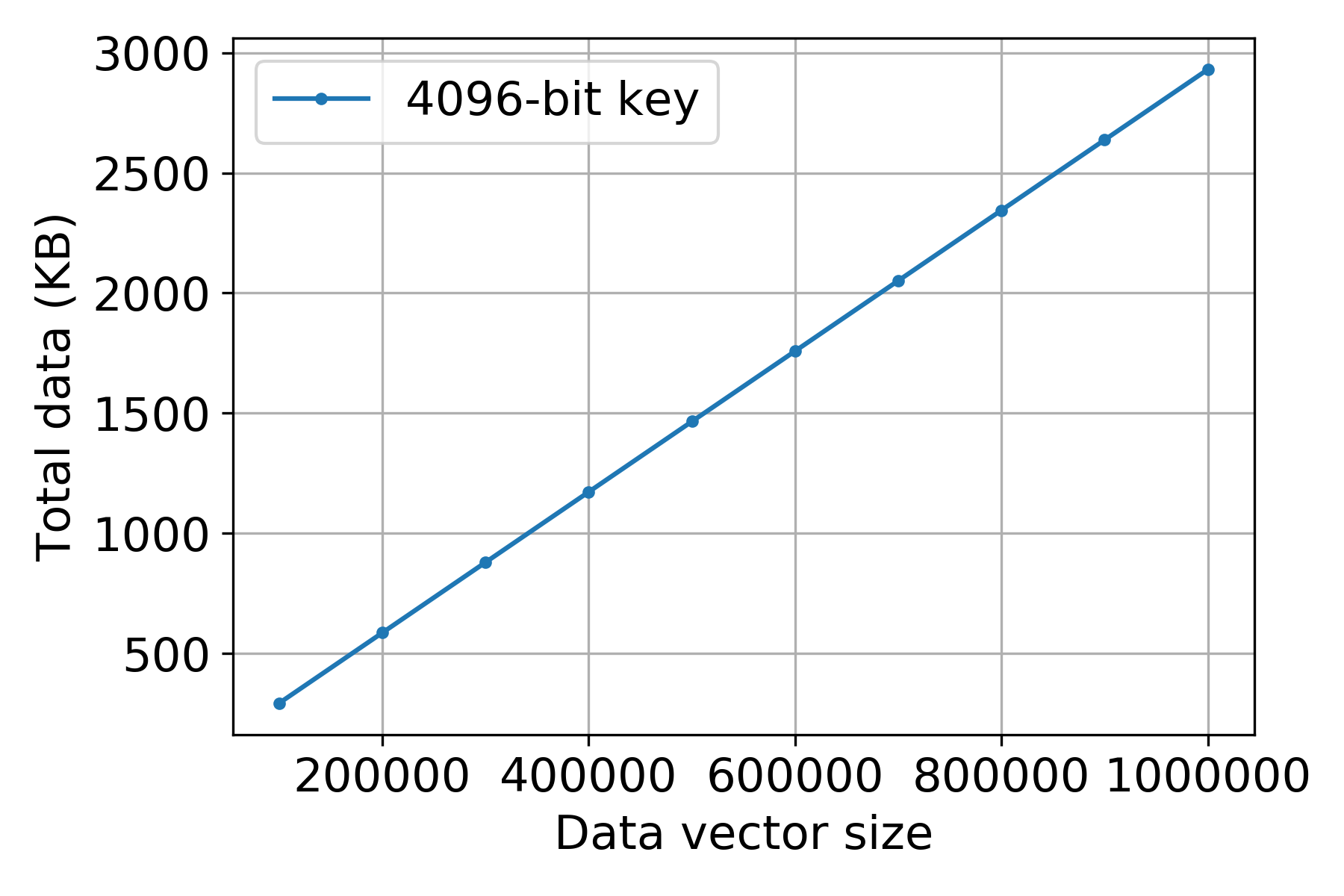}
	    \label{fig:size}
    }
    \caption{Bandwidth consumption per user}
\end{figure}

\noindent\textbf{Key generation.} We first discuss the overhead of distributed key generation. Note that this process can be included during the one-time setup phase of the FL system. An implementation of this phase can be adopted from \cite{das2022practical}. It is shown that with  128 users, the running time for key generation per user is less than 50 seconds, and each user consumes about 15MB of network bandwidth. Therefore, the key generation phase imposes negligible overhead on the system.


\noindent\textbf{Users.} Wall-clock running times for users is plotted in Fig. \ref{fig:time}. We measure the running time with key sizes of 1024, 2048, and 4096, respectively. As can be seen, under all settings, the result conforms with the running time complexity analysis. Also, note that the computation of each user is independent of each other, hence, the number of users does not affect the user's running time. By polynomial packing the data vector and optimizing the encryption process, the user can encrypt the entire data vector within seconds.

Fig. \ref{fig:factor} shows the bandwidth expansion factor of the model encryption per user as we increase the data vector size. It can be observed that the factor approaches 1 as the model size increases, thereby conforming to the analysis of reducing the ciphertext size. Moreover, the number of users also does not affect the bandwidth consumption. This result implies that our secure aggregation scheme induces minimal overhead as compared to sending the raw data vector, and can scale to a large number of users. 

Fig. \ref{fig:size} illustrates the total data transfer of each user. The total data increases linearly with the data vector size, which follows the communication complexity $O(m)$. Furthermore, as shown in Fig. \ref{fig:factor}, the amount of transmitted data is very close to the raw data size under all key sizes, therefore, the total data transfer is approximately the vector size multiplied by 3 bytes. 



\begin{figure}
    \centering
	\subfloat[Wall-clock running time of the server at different data vector sizes, as the number of users increases.]{
	    \includegraphics[width=0.46\linewidth]{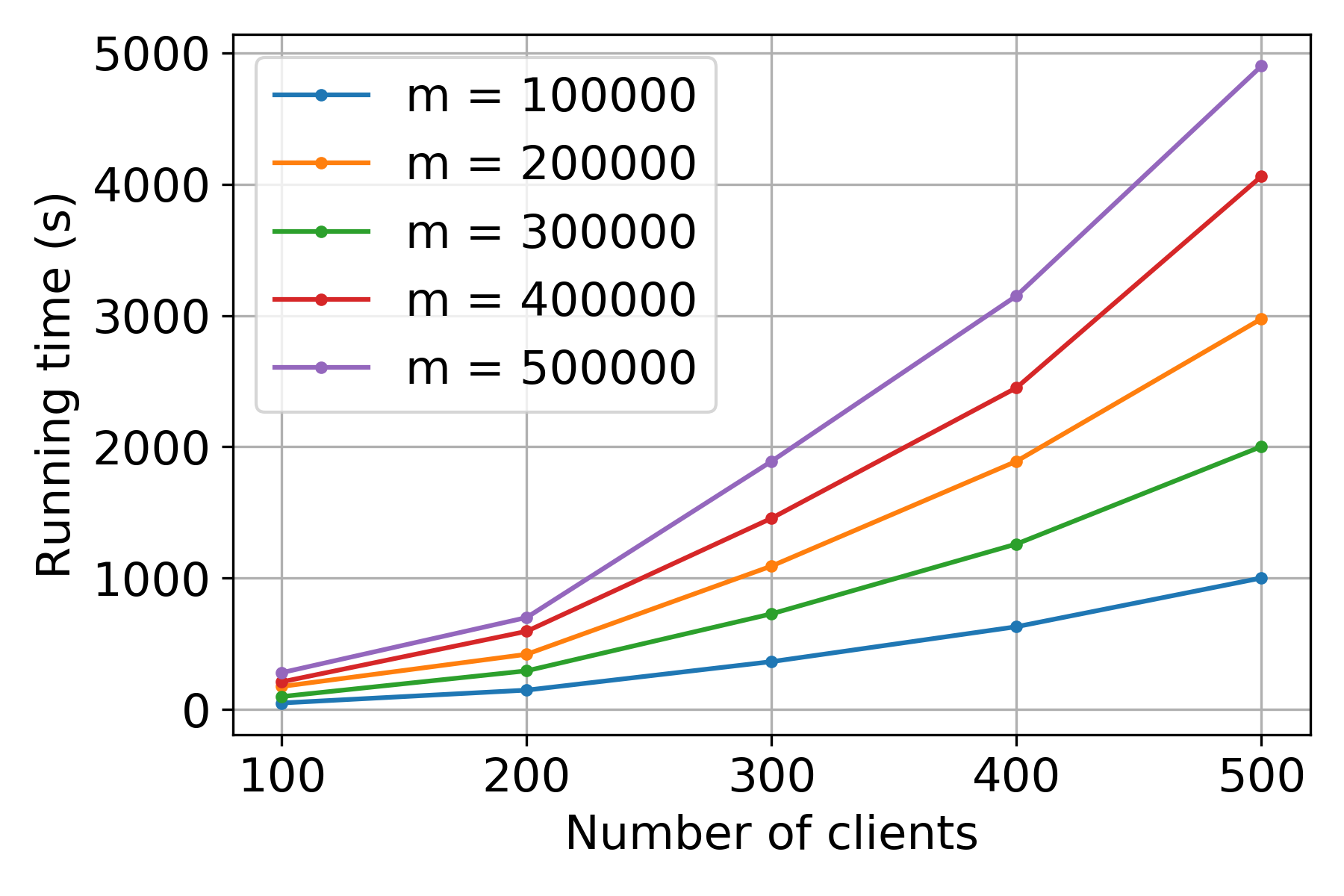}
    	\label{fig:servertime}
	}%
	\hfill
	\subfloat[Total bandwidth consumption of the server at various numbers of users, as the data vector size increase.]{
	    \includegraphics[width=0.46\linewidth]{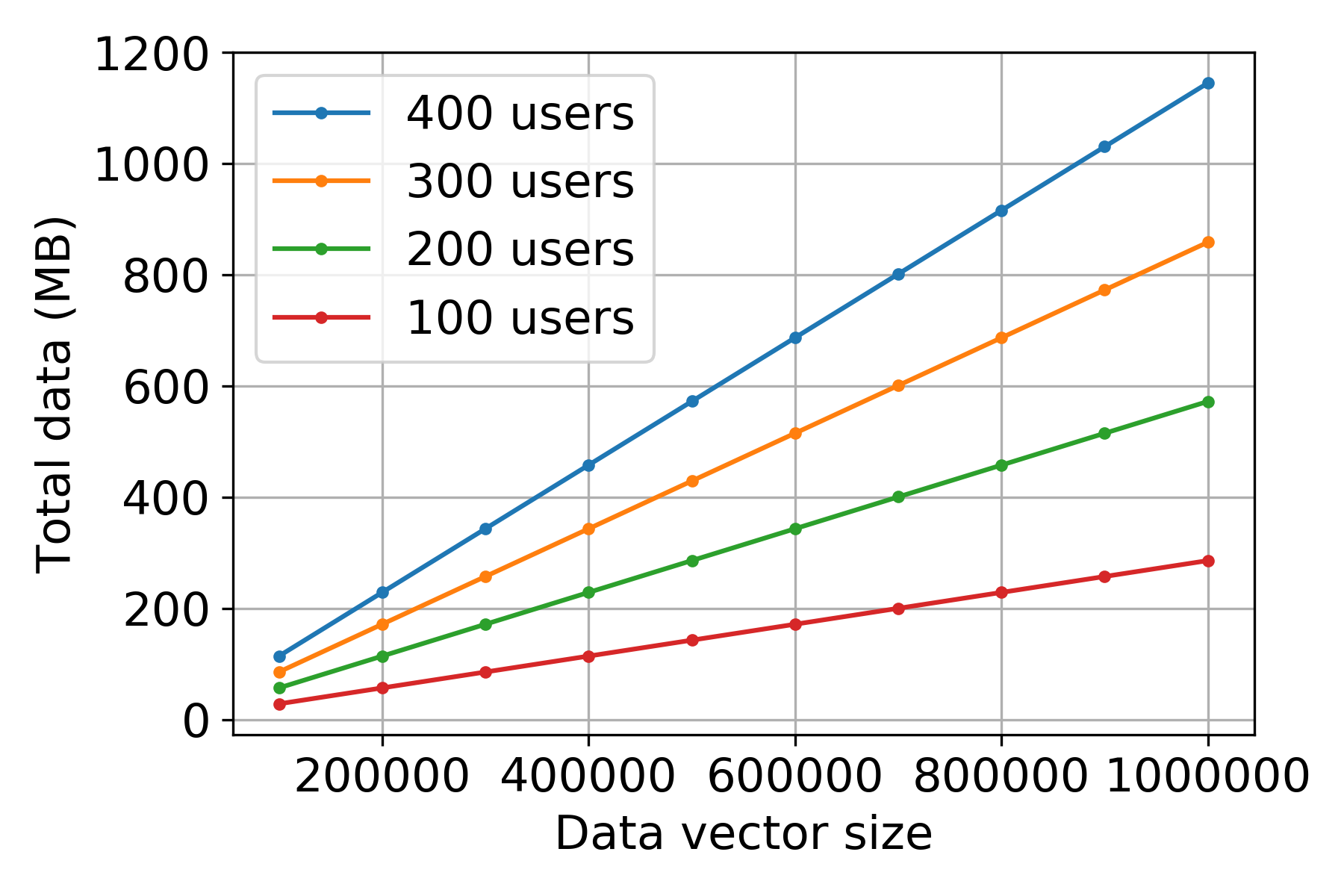}
	    \label{fig:serversize}
    }
    \caption{Running time and bandwidth consumption of the server. The key size is fixed to 4096 bits}
\end{figure}


\noindent\textbf{Server.} On the server's side, Fig. \ref{fig:servertime} shows the running time on the server as we increase the number of users. As can be seen, with 100 users, the server only needs less than 6 minutes to perform decryption, and the result conforms to the running time complexity of $O(m|\calU|^2)$. Note that our testing machine (Intel Core i7, 16GB RAM) has much less computation power than state-of-the-art servers.
 Furthermore, Fig. \ref{fig:serversize} gives the amount of data received by the server. Since we already optimized the size of data sent by users, the bandwidth consumption of the server is also optimized, which follows the communication complexity $O(m|\calU|)$.




\begin{figure}
    \centering
	\subfloat[User generating ZKP]{
	    \includegraphics[width=0.46\linewidth]{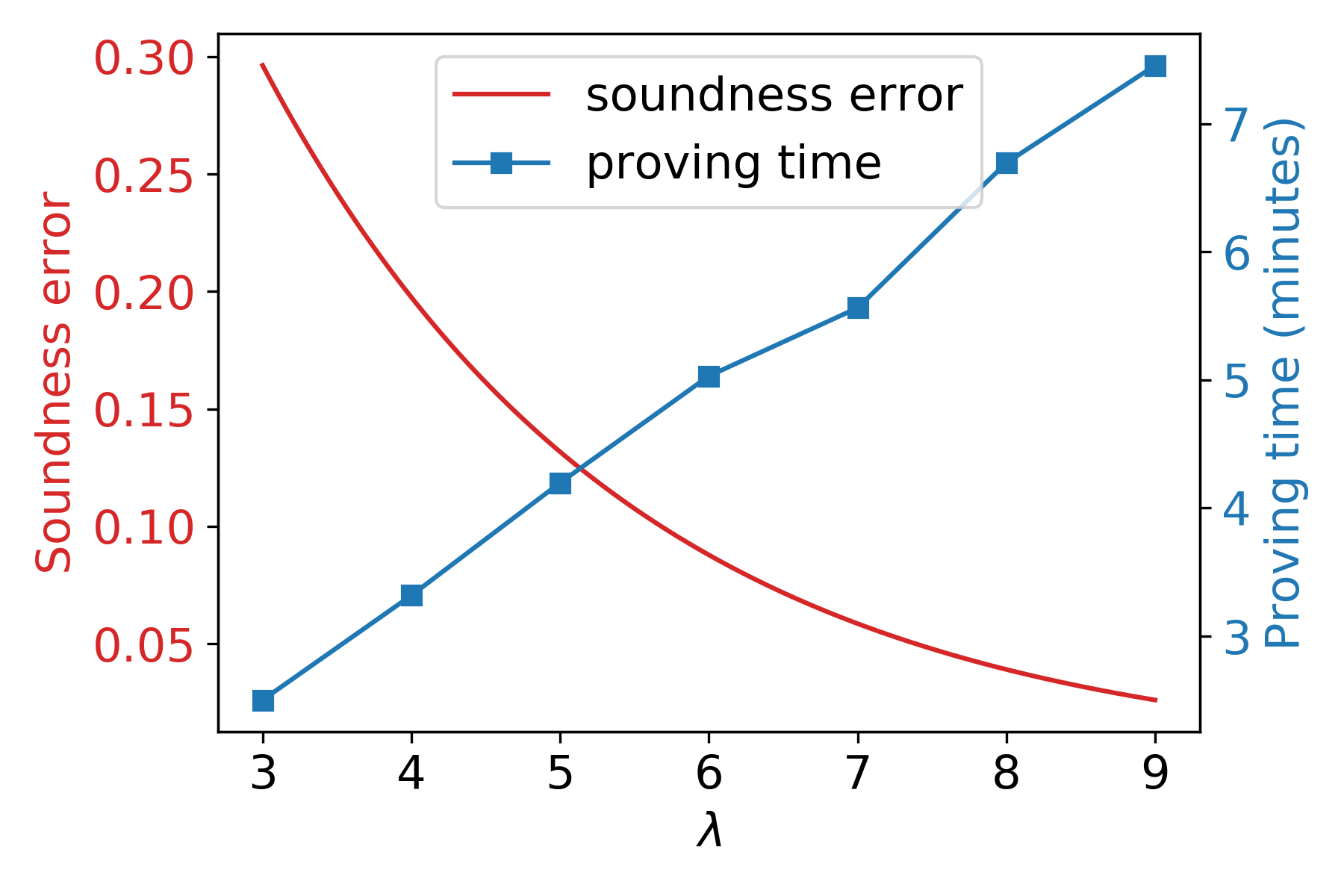}
    	\label{fig:sound}
	}%
	\hfill
	\subfloat[Server validating ZKPs of 100 users]{
	    \includegraphics[width=0.46\linewidth]{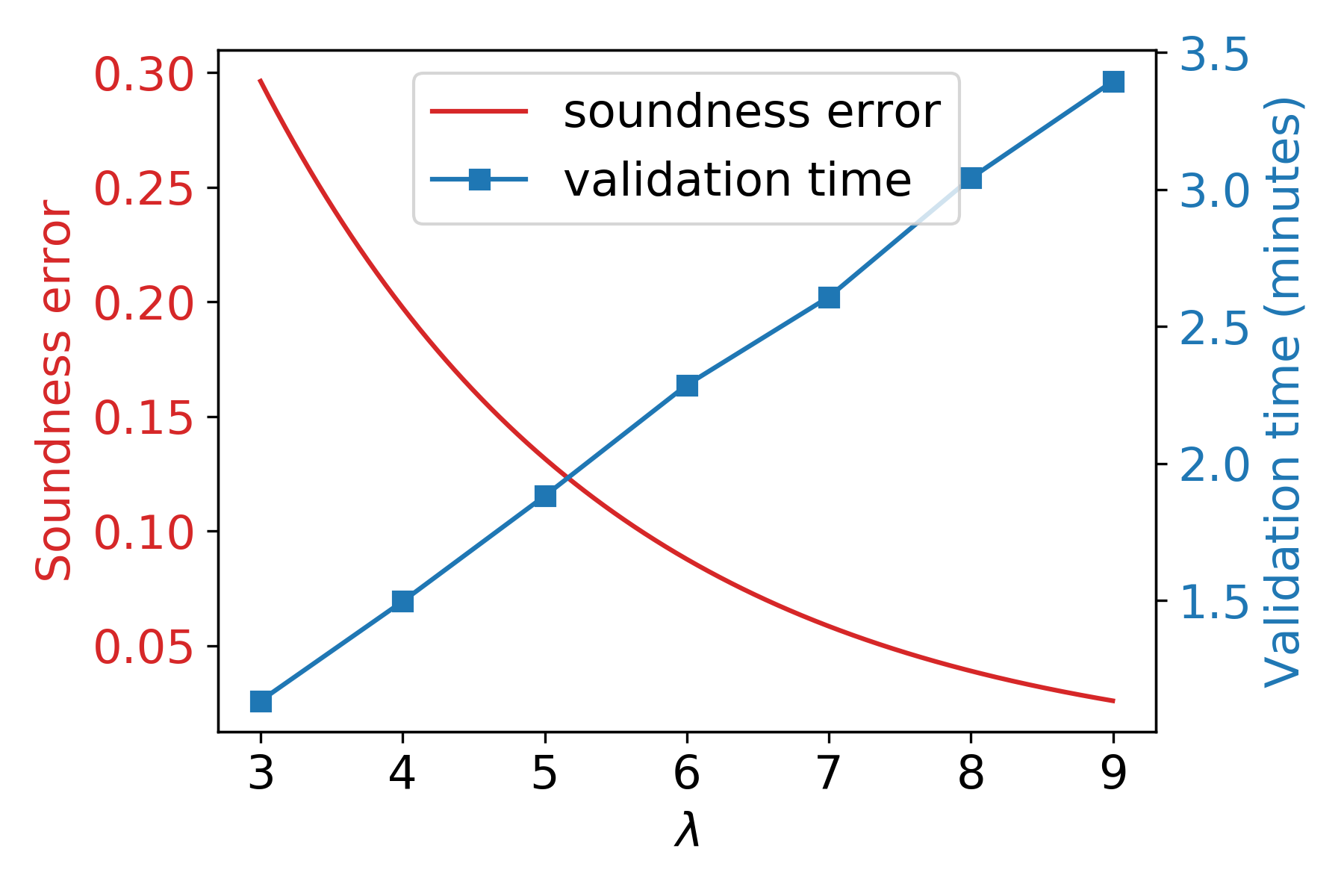}
	    \label{fig:serverval}
    }
    \caption{Wall-clock running time of the zero-knowledge proof protocol on both users and server, as $\lambda$ increases}
    \label{fig:normal}
\end{figure}

\subsubsection{Zero-knowledge proof} 
For the implementation of the zero-knowledge backdoor detection, we train a neural network to use as an input. We use the MNIST dataset and train a 3-layer DNN model in the Tensorflow framework using ReLU activation at the hidden layers and softmax at the output layer. The training batch size is 128. Then, we implement an MPC protocol that is run between 3 emulated parties based on \cite{wagh2019securenn}. The model's accuracy is 93.4\% on the test set.

\noindent\textbf{Users.} To benchmark the performance of the ZKP protocol, we measure the proving and validation time against the soundness error. As shown before, the soundness error is $\frac{2}{3}^{\lambda}$ where $\lambda$ is the security parameter. Fig. \ref{fig:sound} illustrates the proving time to generate a ZKP by each user. We can see that as $\lambda$ increases, the soundness error decreases, nonetheless, the proving time increases linearly. This is a trade-off since the protocol runs slower if we need to obtain a better soundness error. It only takes more than 5 minutes to achieve a soundness error of 0.06, and roughly 7 minutes to achieve 0.03.


\noindent\textbf{Server.} Fig. \ref{fig:serverval} shows the validation time of the ZKP protocol when the server validates the proofs of 100 users. Similar to the proving time, there is a trade-off between the running time and the soundness error. For a soundness error of 0.03, we need about 3.5 minutes to validate 100 users, which is suitable for most FL use cases \cite{bonawitz2019towards}.

To ensure that the ZKP has minimal impact on the model performance, we have also evaluated our framework using CIFAR-10 \cite{krizhevsky2009learning} and Imagenette \cite{howard2020imagenette}. After training the models on these datasets under our framework with FL, we attain an accuracy of 86.6\% and 92.5\% on CIFAR-10 and Imagenette, respectively. This shows that our framework imposes negligible trade-offs in terms of model accuracy.


\section{Related work} \label{sec:rel}
\noindent\textbf{Secure aggregation in FL.} Leveraging secret sharing and random masking, Bonawitz et al. \cite{bonawitz2017practical} propose a secure aggregation protocol and utilize it to aggregate local models from users. However, the protocol relies on users honestly following the secret sharing scheme. Consequently, although the privacy guarantee for honest users still retains, a single malicious user can arbitrarily deviate from the secret sharing protocol and make the server fail to reconstruct the global model. Furthermore, there is no mechanism to identify the attacker if such an attack occurs.

In \cite{aono2017privacy} and \cite{zhang2020batchcrypt}, the decryption key is distributed to the users, and the server uses homomorphic encryption to blindly aggregate the model updates. However, the authors assume that there is no collusion between the server and users so that the server cannot learn the decryption key. Hence, the system does not work under our security model where users can be malicious and collude with the server.

Additionally, there have been studies on how to use generic secure MPC based on secret sharing to securely compute any function among multiple parties \cite{damgaard2012multiparty,ben2019completeness,lindell2015efficient}. However, with these protocols, each party has to send a secret share of its whole data vector to a subset of the other parties. In order to make the protocols robust, the size of this subset of users should be considerably large. Since each secret share has the same size as the entire data vector's, these approaches are not practical in federated learning in which we need to deal with high-dimensional vectors.

On the other hand, differential privacy (DP) is considered as a prominent privacy tool with a formal guarantee on the privacy leakage \cite{dwork2014algorithmic,nguyen2023xrand}. However, DP only aims to protect membership privacy and does not conceal the local models from the server. Previous work has shown that the users in FL are still susceptible to other privacy attacks, such as data reconstruction, even when DP is used \cite{boenisch2021curious,fowl2021robbing}. Furthermore, DP comes with a trade-off that reduces the model performance. It has been suggested that using DP makes it impossible to train a good model for datasets like CIFAR-10 or ImageNet with reasonable accuracy \cite{boenisch2021curious,tramer2020differentially}. Our proposed protocol fully hides the local models from the server, disabling any attacks that rely on inspecting the local models, while maintaining the performance of the models.



\noindent\textbf{Zero-knowledge proof for detecting poisoned models.} Although there have been many studies on developing a ZKP protocol based on MPC, they have not been widely used in the machine learning context. The IKOS protocol proposed by Ishai et al. \cite{ishai2007zero} is the first work that leverages secure multi-party computation protocols to devise a zero-knowledge argument. Giacomelli et al. \cite{giacomelli2016zkboo} later refine this approach and construct ZKBoo, a zero-knowledge argument system for Boolean circuits using collision-resistant hashes that does not require a trusted setup. Our proposed ZKP protocol is inspired by the IKOS protocol with repetitions to reduce the soundness error to $\frac{2}{3}^{\lambda}$. We also demonstrate how such protocol can be used in the context of machine learning, especially for attesting non-poisoned models while maintaining privacy guarantees.

Another line of research in ZKP focuses on zkSNARK implementations \cite{bitansky2013succinct} that have been used in practical applications such as ZCash \cite{sasson2014zerocash}. However, these systems depend on cryptographic assumptions that are not standard, and have large overhead in terms of memory consumption and computation cost, thereby limiting the statement sizes that they can manage. Therefore, it remains a critical challenge whether zkSNARK could be used in machine learning where the circuit size would be sufficiently large.


\noindent\textbf{Defense mechanisms against poisoning attacks.} There have been multiple research studies on defending against poisoning attacks. Liu et al. propose to remove backdoors by pruning redundant neurons \cite{liu2018fine}. On the other hand, Bagdasaryan et al. \cite{bagdasaryan2020backdoor} devise several anomaly detection methods to filter out attacks. More recently proposed defense mechanisms \cite{liu2019abs,wang2019neural} detect backdoors by finding differences between normal and infected label(s).

Our work leverages such a defense mechanism to construct a ZKP protocol for FL in which the users can run the defense locally on their model updates, and attest its output to the server, without revealing any information about the local models. Different defense mechanisms may have different constructions of the ZKP protocols, nevertheless, they must all abide by the properties specified in our framework (Section \ref{sec:frame}). As shown in Section \ref{sec:zkp}, we construct a ZKP protocol based on the defense proposed by Liu et al. \cite{liu2018fine}.

\section{Conclusion} \label{sec:con}
In this paper, we have proposed a secure framework for federated learning. Unlike existing research, we integrate both secure aggregation and defense mechanisms against poisoning attacks under the same threat model that maintains their respective security and privacy guarantees. We have also proposed a secure aggregation protocol that can maintain liveness and privacy for model updates against malicious users. Furthermore, we have designed a ZKP protocol for users to attest non-backdoored models without revealing any information about their models. Our framework combines these two protocols and shows that the server can detect backdoors while preserving the privacy of the model updates. The privacy guarantees for the users' models have been theoretically proven. Moreover, we have presented an analysis of the computation and communication cost and provided some benchmarks regarding its runtime and bandwidth consumption

\section*{Acknowledgements}
This material is based upon work supported by the National Science Foundation under grants CNS-1935923 and CNS-2140477.

\bibliographystyle{plain}
\bibliography{bibliography}

\clearpage
\appendices

\section{Proof for Theorem 1} \label{proof:1}
We note that the simulator can produce a perfect simulation as follows: it runs the malicious users on their true inputs, and all other users on random bitstrings, and then outputs the simulated view of the users in ${Adv}$. This is because the combined view of the users in ${Adv}$ does not depend on the inputs of the users who are not in ${Adv}$. In particular, from the User aggregation component, we obtain a PPT simulator that can simulate the view of honest users. Next, from the Zero-knowledge proof component, we can build a simulator upon dummy input (since it does not have access to the prover) such that the output of $SIM$ is computationally indistinguishable from the output $Adv$.

\section{Proof for Theorem 3} \label{proof:3}
By pruning the neuron with the lowest average activation in each iteration, the defense mechanism, in fact, comprises 3 phases of pruning neurons \cite{liu2018fine}:

\begin{enumerate}
    \item In the first phase, the pruned neurons do not affect the model accuracy or the attack success rate of the backdoor.
    \item In the second phase, it prunes neurons that are activated by the backdoor but not by clean inputs. At this time, the success rate of the backdoor attack is reduced while the model accuracy is maintained.
    \item  The third phase prunes neurons that are activated by clean inputs, thereby reducing the model accuracy.
\end{enumerate}

Hence, as we are pruning the neuron with minimum average activation in Algorithm \ref{algo:backdoor}, the accuracy only drops if the backdoored neurons no longer exist according to the pruning defense.

\section{Proof for Theorem 4} \label{proof:4}

\noindent \textbf{Completeness:} Since $x = x_1 + x_2 + x_3$ and $\Pi_f$ are perfectly correct, the views $v_1, v_2, v_3$ always output the commitment of $x$. As these views are generated by an honest execution of the protocol, they remain consistent with one another.

\noindent \textbf{Soundness:} (if $x$ is backdoored then the verifier accepts the proof with probability at most $\epsilon$ where $\epsilon$ is the soundness error): First, since the commitments generated in step 5 are included in the input query to the random oracle, that means the prover does not know $e_1$ and $e_2$ prior to step 6, hence it does not know which views are to be revealed before committing (if it does, it would create 2 fake views to always pass all the checks). Additionally, since the commitment scheme is computationally binding, the views $v_{e_1}$ and $v_{e_2}$ must be the same as the ones in step 4, or else they would be rejected by the first check.

Since $x = x_1 + x_2 + x_3$ and $\Pi_f$ is perfectly correct, suppose the prover is honest, an honest execution will produce the views $v_1, v_2, v_3$ whose output is $\bot$. Therefore, taking into account the 3 views committed to by the prover in Step 4, either (1) in all views the output is $\bot$ (the prover is honest), or (2) there exists two views that are inconsistent (the prover is dishonest) by the above-mentioned property. 

In case (1), the proof is rejected by the second condition with probability 1. In case (2), the verifier rejects with probability at least $1/{3 \choose 2}$, which is the probability of selecting an inconsistent pair of views. That means the soundness error is $1 - 1/{3 \choose 2} = \frac{2}{3}$.





As we are running $\lambda$ iterations, at each replication, the verifier accepts the proof independently with probability at most $\frac{2}{3}$. Hence, the probability that the verifier accepts all proof is at most $\frac{2}{3}^{\lambda}$


\textbf{Zero-knowledge:} The verifier only has access to a subset of 2 views, since we are using a 3-party 2-private MPC protocol, these 2 views do not reveal anything about the private input $x$ (the privacy properties of the MPC protocol). Thus, the verifier learns nothing about $x$.

In fact, we could also construct a perfect simulation for the verifier where the simulator is the same as the one described in \cite{ishai2007zero}, which also uses a 2-private MPC protocol.

\end{document}